\documentclass[nohyperref]{article}

\usepackage{microtype}
\usepackage{graphicx}
\usepackage{subfigure}
\usepackage{booktabs} 
\usepackage{algorithmic}
\usepackage[ruled,vlined]{algorithm2e}
\usepackage{amsfonts}
\usepackage{bbm}

\usepackage{hyperref}

\newcommand{\R}{{\mathbb{R}}}

\newcommand{\E}{{\mathbb{E}}}

\newcommand{\bbP}{{\mathbb{P}}}
\newcommand{\bbI}{{\mathbbm{1}}}    
\newcommand{\N}{{\mathbb{N}}}

\usepackage[accepted]{icml2022}

\usepackage{amsmath}
\usepackage{amssymb}
\usepackage{mathtools}
\usepackage{amsthm}

\usepackage[capitalize,noabbrev]{cleveref}

\theoremstyle{plain}
\newtheorem{theorem}{Theorem}[section]

\newtheorem{lemma}[theorem]{Lemma}
\newtheorem{corollary}[theorem]{Corollary}
\theoremstyle{definition}

\theoremstyle{remark}
\newtheorem{remark}[theorem]{Remark}

\icmltitlerunning{Convergence Guarantees for Deep Epsilon Greedy Policy Learning}

\begin{document}

\twocolumn[
\icmltitle{Convergence Guarantees for Deep Epsilon Greedy Policy Learning}



\icmlsetsymbol{equal}{*}

\begin{icmlauthorlist}
\icmlauthor{Michael Rawson}{equal,umd}
\icmlauthor{Radu Balan}{equal,umd}
\end{icmlauthorlist}

\icmlaffiliation{umd}{Department of Mathematics, University of Maryland, College Park, Maryland, USA}

\icmlcorrespondingauthor{Michael Rawson}{rawson@umd.edu}

\icmlkeywords{Machine Learning, ICML}

\vskip 0.3in
]



\printAffiliationsAndNotice{\icmlEqualContribution} 

\begin{abstract}
Policy learning is a quickly growing area. As robotics and computers control day-to-day life, their error rate needs to be minimized and controlled. There are many policy learning methods and bandit methods with provable error rates that accompany them. We show an error or regret bound and convergence of the Deep Epsilon Greedy method which chooses actions with a neural network's prediction. We also show that Epsilon Greedy method regret upper bound is minimized with cubic root exploration. In experiments with the real-world dataset MNIST, we construct a nonlinear reinforcement learning problem. We witness how with either high or low noise, some methods do and some do not converge which agrees with our proof of convergence. 
\end{abstract}

\section{Introduction and Related Work}
In recent years, computer automation has taken control over many processes and tasks. Researchers have searched for the best methods for computers to optimize their performance. We call a method that chooses actions a `policy'. Many methods are shown to converge to the best possible policy over time with various assumptions, like the bandit assumption. The Epsilon Greedy method is one of the oldest methods \cite{sutton_reinforcement_1998}. And when the state/context of a system is constant, Epsilon Greedy is known to converge to the optimal policy \cite{auer_finite-time_2002}. Other methods like Upper Confidence Bound (UCB) and Thompson Sampling also converge (\cite{auer_finite-time_2002}, \cite{agrawal_analysis_2012}). When the state/context may change, the problem is more challenging. Many papers, including this one, will assume that the state is a random variable sampled at each time step and independent of past actions and rewards. This is known as \emph{unconfoundedness} \cite{athey_policy_2020}. And in (\cite{athey_policy_2020}, \cite{xu2021doubly}), convergence is shown for the Doubly Robust method. When state $X$ is sampled on a low dimensional manifold, fast convergence is shown in \cite{chen_doubly_2020}. Other neural network based policy learning methods also converge \cite{zhou2020neural}, \cite{rawson_deep_2021}. We will show convergence for Deep Epsilon Greedy method under reasonable assumptions and then discuss variations and generalizations. The convergence of policy learning depends on the policy and the neural network. The neural network's convergence that we use goes back to \cite{gyorfi_distribution-free_2002} which bounds the neural network's parameter weights. This is equivalent to a Lipschitz bound on the employed class of neural networks \cite{zou2019lipschitz}.

\section{Algorithm}
First we describe the well known Epsilon Greedy method in Algorithm \ref{algo:deep_ep_greedy}. We will use and analyze this algorithm throughout this paper. This method runs for $M$ time steps and at each time step takes in a state vector, $X_t$, and chooses an action, $D_t$, from $A=\{action_1,...,action_K\}$. The reward at each time step is recorded and the attempt is to maximize the total rewards received. 

\begin{algorithm}[t]
\SetAlgoLined
    \textbf{Input:}\\
    $M \in \mathbb{N}$ : Total time steps \\
    $m \in \mathbb{N}$ : Context dimension \\
    $ X \in \mathbb{R}^{M \times m}$ where state $X_{t}\in\mathbb{R}^m$ for time step $t$  \\ 
    $ A = \{action_1,...,action_K\}$ : Available Actions \\
    $ \Phi :\mathbb{R}^m \rightarrow \mathbb{R} $ : Untrained Neural Network \\
    $ Reward : \mathbb{N}_{[1,K]} \rightarrow \mathbb{R}$  \\
    \textbf{Output:}\\
    $ D \in \N^M $ : Decision Record \\
    $ R \in \mathbb{R}^{M}$ where $R_{t}$ stores the reward from time step $t$  \\ 
    \textbf{Begin:}\\
    \For{t = 1, 2, ..., $M$}{
        \For{j = 1 ... K}{
            $ \hat \mu_{action_j} = \Phi_{j,t}(X_{t}) $\ (predict reward) 
        }
        $ \eta \sim $ Uniform(0,1)  \\
        $\epsilon_t = 1/t$ \\
        \eIf{$\eta > \epsilon_t$}{
            $D_{t} =  \arg \max_{1 \le j \le K } \ \hat \mu_{action_j} $ 
        }{
            $ \rho \sim $ Uniform(\{1,...,K\})  \\
            $D_{t} = A_{\rho}$ 
        }
        $R_{t} = Reward(D_{t})$ \\
        (Training Stage) \\
        \For{j = 1 ... K}{
            $ S_j = \{l : 1\le l\le t,\ D_l = j\}$ \\
            TrainNNet($\Phi_{j,t-1}, input=X_{S_j}, output=R_{S_j} $)
        }
    }
\caption{Deep Epsilon Greedy}
\label{algo:deep_ep_greedy}
\end{algorithm}

\section{Results}

\begin{theorem} 
\label{nnet} [\cite{gyorfi_distribution-free_2002} Theorem 16.3]
Let $\Phi_n$ be a neural network with $n$ parameters and the parameters are optimized to minimize MSE of the training data, $S = \{(X_i,Y_i)\}$ where $X$ and $Y$ are almost surely bounded. 
Let the training data, be size n, and $Y_i = R(x_i) \sim N(\mu_{x_i}, \sigma_{x_i})$ where 
$R : \R^m \rightarrow \R$. 
Then for $n$ large enough, 
\begin{align} \label{eq:nnet}
 \mathbb{E}_S \int | \Phi_n(x) - \mathbb{E}(R(x)) |^2 dP(x) \le c\sqrt{\frac{\log(n)}{n}} 
\end{align}
for some $c > 0$.

\end{theorem}

Assume there are $K$ actions to play. Let $T_j(t)$ be the random variable equal to number of times action $j$ is chosen in the first $t-1$ steps.
Let $T_j^R(t)$ be the number of times action $j$ is chosen in the first $t-1$ steps by the uniform random branch of the algorithm.
Let $X$ be the state vector at some time step $t$ and $Y^j_t$ be the reward of action $j$ at time step $t$ both almost surely bounded. Let $\mu_j(X) := \mathbb{E}(Y^j_t|X)$.
We'll use $*$ for an optimal action index, for example let $\mu_*(X)$ be the expectation of all optimal actions at $X$. 
Let $\Delta_j(X) := \max\{0,\mu_*(X) - \mu_j(X)\}$. 
Let $\epsilon_t = 1/t$. Let $I_t$ be the action chosen at time $t$.
Assume state $X$ is sampled from an unknown distribution i.i.d. at each time step $t$. 

\begin{theorem} 
\label{thm:epsilon_greedy_bound} Assume there is optimality gap $\delta$ with $0 < \delta \le \Delta_j(X)$ for all $j$ and $X$ where $j$ is suboptimal. Assume there is at least one suboptimal action for any context. With the assumptions from above and from Theorem \ref{nnet}, the Deep Epsilon Greedy method converges with expected regret approaching 0 almost surely. Let $C_i$ be the constant from Theorem \ref{nnet} for neural network $i$ and let $n_i$ be the minimal value of the training data size such that equation \ref{eq:nnet} holds. Set
$C_0=8\sqrt{2}\max_i C_i$ and $t_0 = \exp(2K \max\{e,\max_i n_i\})$. Then for every $t > t_0$ with probability at least $1-K\exp(-3 \ln(t)/(28K))$,
\begin{align}
& \delta / (tK) \le \E_{X_t} \E_{I_t} \E_R \left[ R_*(X_t) - R(X_t) \right] \label{eq:mainth1}\\  
& \le \frac{\max_i \E_{X_t} \Delta_i(X_t)}{t} 
 + K^{3/2}\frac{C_0}{\delta} \sqrt{\frac{\ln(\ln(t))-\ln(2K)}{\ln(t)}}. \label{eq:mainth2}
\end{align}
The expectations in above equations refer to the specific time step $t$. The probability refers to the stochastic policy's choices at previous time steps, 1 to $t-1$. 
\end{theorem}

\begin{theorem} 
\label{thm:epsilon_greedy_bound_p}
Let $\epsilon_t = 1/t^p$ where $0<p<1$. Assume the assumptions of Theorem \ref{thm:epsilon_greedy_bound}. Set \\
$C_0'=8\sqrt{2(1-p)}\max_i C_i$ and \\
$t_0>(2 (1-p) K \max\{e,\max_i n_i\})^{1/(1-p)}$. Then for every $t > t_0$ with probability at least

$1-K\exp\left(-(3 \ t^{-p+1})/(28 (-p+1) K)\right)$,
\begin{align}
& \delta / (K t^p) \le \E_{X_t} \E_{I_t} \E_R \left[ R_*(X_t) - R(X_t) \right] \label{eq:epsilon_greedy_bound_p1}\\  
& \le \frac{\max_i \E_{X_t} \Delta_i(X_t)}{t^p} \label{eq:epsilon_greedy_bound_p2}\\
& + K^{3/2}\frac{C_0'}{\delta}  
 \sqrt{\frac{\ln(t^{-p+1})-\ln(2(-p+1)K)}{t^{-p+1}}}.
 \label{eq:epsilon_greedy_bound_p3}
\end{align}
The expectations in above equations refer to the specific time step $t$. The probability refers to the stochastic policy's choices at previous time steps, 1 to $t-1$. 
\end{theorem}

Proof in appendix.

\begin{lemma}
\label{lemma:epsilon_greedy_bound}
Recall that $T_j^R(t)$ is the number of times action $j$ is chosen in the first $t-1$ steps by the uniform random branch of the algorithm. For the case $\epsilon_t = 1/t$,
\begin{align} \label{lemma:epsilon_greedy_bound1}
 \bbP & \left(\bigwedge_{i=1}^K \{T^R_i(t) \ge \ln(t)/(2K)\}\right) \\
& \ge 1-K\exp\left(-\frac{3 \ln(t)}{28K}\right). 
\end{align}
For the case $\epsilon_t = 1/t^p$, where $0<p<1$,
\begin{align} \label{lemma:epsilon_greedy_bound2}
 \bbP & \left(\bigwedge_{i=1}^K \{T^R_i(t) \ge\frac{t^{-p+1}}{2(-p+1) K}\}\right) \\
& \ge 1-K\exp\left(-\frac{3 \ t^{-p+1}}{28 (-p+1) K}\right). 
\end{align}
\end{lemma}

\begin{proof}[Proof of lemma \ref{lemma:epsilon_greedy_bound}, part I.]
Fix $i$. Recall $\epsilon_t = 1/t$. Following the proof of theorem 3 in \cite{auer_finite-time_2002}, 
\begin{align*}
 \E( T^R_i(t) ) &= \sum_{l=1}^{t-1} \bbP(\eta < \epsilon_l \wedge \rho = i) 
 = \sum_{l=1}^{t-1} \bbP(\eta < \epsilon_l) \bbP(\rho = i) \\
 &= \sum_{l=1}^{t-1} \epsilon_l/K 
 =\frac{1}{K} \sum_{l=1}^{t-1} \frac{1}{l} 
 \ge \frac{1}{K}\ln(t)
\end{align*}
Set $B(t):=\frac{1}{K} \sum_{l=1}^{t-1} \frac{1}{l} $. Then we have
\begin{align*}
Var(T^R_i(t)) &= \sum_{l=1}^{t-1} \frac{\epsilon_l}{K}(1-\frac{\epsilon_l}{K}) 
\le \frac{1}{K} \sum_{l=1}^{t-1} \epsilon_l \\
&=\frac{1}{K} \sum_{l=1}^{t-1} \frac{1}{l} 
=B(t).
\end{align*}
By Bernstein’s inequality 
\begin{align*}
 \bbP(&T^R_i(t)  \le B(t)/2 ) 
 = \bbP\left(T^R_i(t) - B(t) \le -B(t)/2 \right) \\
&\le \exp\left(\frac{-B(t)^2/8}{Var(T^R_i(t))+\frac{1}{3} B(t)/2}\right)\\ 
&\le \exp\left(\frac{-B(t)^2/8}{B(t)+\frac{1}{3} B(t)/2}\right) \\
&\le \exp\left(-\frac{3 B(t)}{28}\right)
 \le \exp\left(-\frac{3 \ln(t)}{28K}\right). 
\end{align*}
So by union bound
$$ \bbP\left(\bigvee_{i=1}^K \{T^R_i(t) 
 \le \ln(t)/(2K)\}\right) $$
$$ \le K \bbP\left( T^R_1(t) \le \ln(t)/(2K)\right) 
 \le K\exp\left(-\frac{3 \ln(t)}{28K}\right). $$
And
$$ \bbP\left(\bigwedge_{i=1}^K \{T^R_i(t) \ge \ln(t)/(2K)\}\right) 
\ge 1-K\exp\left(-\frac{3 \ln(t)}{28K}\right). $$
This prove Equation \eqref{lemma:epsilon_greedy_bound1}. A similar proof for Equation \eqref{lemma:epsilon_greedy_bound2} is contained in the appendix. 
\end{proof}

\begin{proof}[Proof of theorem \ref{thm:epsilon_greedy_bound}]

Let $*$ be an optimal action at $X_t$ and $R$ the reward from the epsilon greedy method. 


Let $\Phi_{i,t}$ be the trained neural network for action i and have t parameters.
By lemma \ref{lemma:epsilon_greedy_bound}, with probability greater than $1-K\exp(-3 \ln(t)/(28K))$, $T^R_i(t) \ge \ln(t)/(2K)$ for all $i$. In this case, we have 

\begin{align*}
 \E_{X_t} & \E_{I_t} \E_R[R_*(X_t) - R(X_t)] \\
&  = \E_{X_t} [\mu_*(X_t) - \E_{I_t} \E_R R(X_t) ] \\
& = \E_{X_t} [\mu_*(X_t) - \sum_{i=1}^K \mu_i(X_t) \bbP(I_t=i | X_t) ] \\
& = \E_{X_t} \sum_{i} \Delta_i(X_t) \bbP(I_t=i | X_t) \\
& = \sum_{i} \E_{X_t} \Delta_i(X_t) \bbP(I_t=i | X_t) 
\end{align*}

Then
\begin{align*}
 \E_{X_t} [ & \Delta_i(X_t) \bbP(I_t=i | X_t) ] \\
 & \le \E_{X_t} \Delta_i(X_t) [\epsilon_t/K + \bbP(\Phi_{i,t}(X_t) \ge \Phi_{*,t}(X_t))] 
\end{align*}

and, with Markov's inequality,
\begin{align*}
\bbP ( &\Phi_{i,t}(X_t) \ge \Phi_{*,t}(X_t) ) \\
&\le \bbP(\Phi_{i,t}(X_t) \ge \mu_i(X_t) + \Delta_i(X_t)/2) + \\ 
& +\bbP(\Phi_{*,t}(X_t) \le \mu_*(X_t) - \Delta_i(X_t)/2) \\
& = \int_{\bbI\{\Phi_{i,t}(X_t) \ge \mu_i(X_t) + \Delta_i(X_t)/2\}} dP_i + \\
& + \int_{\bbI\{\Phi_{*,t}(X_t) \le \mu_*(X_t) - \Delta_i(X_t)/2\}} dP_* \\
& \le \int_{\bbI\{|\Phi_{i,t}(X_t) - \mu_i(X_t)| \ge \Delta_i(X_t)/2\}} dP_i + \\
& + \int_{\bbI\{|\Phi_{*,t}(X_t)-\mu_*(X_t)| \ge \Delta_i(X_t)/2\}} dP_* \\
& \le \int \frac{|\Phi_{i,t}(X_t) - \mu_i(X_t)|^2 }{\Delta_i(X_t)^2/4} dP_i \\
& +  \int \frac{|\Phi_{*,t}(X_t)-\mu_*(X_t)|^2 }{\Delta_i(X_t)^2/4} dP_* 
\end{align*}

Then
\begin{align*}
 \E_{X_t} [ & \Delta_i(X_t) \bbP(I_t=i | X_t) ] \\
& \le \E_{X_t} \Delta_i(X_t) \epsilon_t/K + \\
& + \E_{X_t} \Delta_i(X_t) \int \frac{|\Phi_{i,t}(X_t) - \mu_i(X_t)| }{\Delta_i(X_t)/2} dP_i + \\
& + \E_{X_t} \Delta_i(X_t) \int \frac{|\Phi_{*,t}(X_t)-\mu_*(X_t)| }{\Delta_i(X_t)/2} dP_* \\
\\
& \le \E_{X_t} \Delta_i(X_t) \epsilon_t/K \\
& + \frac{4}{\delta} \int_{x_t : i\ne *} \int |\Phi_{i,t}(x_t) - \mu_i(x_t)|^2 dP_i dP_{x_t}  + \\
& + \frac{4}{\delta} \int_{x_t : i\ne *} \int |\Phi_{*,t}(x_t)-\mu_*(x_t)|^2 dP_* dP_{x_t},
\end{align*}
by dominated convergence,
\begin{align*}
& \le \E_{X_t} \Delta_i(X_t)\ \epsilon_t/K \\
& + \frac{4}{\delta} \int \int_{x_t } |\Phi_{i,t}(x_t) - \mu_i(x_t)|^2 dP_{x_t} dP_i + \\
& + \frac{4}{\delta} \int \int_{x_t } |\Phi_{*,t}(x_t)-\mu_*(x_t)|^2 dP_{x_t} dP_* 
\end{align*}

Recall that we are in the case that $T^R_i(t) \ge \ln(t)/(2K)$ for all $i$. Let $C_i$ be the constant from Theorem \ref{nnet} for neural network $i$ and let $n_i$ be the minimal value of the training data size such that equation \ref{eq:nnet} holds. Choose $t_0>\exp(2 K \max\{e,\max_i n_i\})$. Since the map $x \mapsto \sqrt{\frac{ln(x)}{x}}$ is monotone decreasing for $x>e$, the above expression is further upper bounded by 
\begin{align*}
& \le \E_{X_t} \Delta_i(X_t)\ \frac{\epsilon_t}{K} 
 + \frac{4}{\delta}  C_i\sqrt{\frac{\ln(T_i(t))}{T_i(t)}} 
 + \frac{4}{\delta}  C_*\sqrt{\frac{\ln(T_*(t))}{T_*(t)}}  \\
& \le \E_{X_t} \Delta_i(X_t)\ \frac{\epsilon_t}{K}
 + \frac{4C_i}{\delta}  \sqrt{\frac{\ln(T^R_i(t))}{T^R_i(t)}} 
 + \frac{4C_*}{\delta}  \sqrt{\frac{\ln(T^R_*(t))}{T^R_*(t)}}  \\
& \le \frac{\E_{X_t} \Delta_i(X_t)}{tK} 
 + \left[\frac{4C_i}{\delta}  
 + \frac{4C_*}{\delta}\right] \sqrt{\frac{\ln(\ln(t)/(2K))}{\ln(t)/(2K)}}
\end{align*}
So
\begin{align*}
& \E_{X_t} \E_{I_t} \E_R [R_*(X_t) - R(X_t)]\\
& = \E_{X_t} [\mu_*(X_t) - \E_{I_t} \E_R R(X_t) ] \\
& = \sum_{i} \E_{X_t} \Delta_i(X_t) \bbP(I_t=i | X_t) \\
& \le \frac{\max_i \E_{X_t} \Delta_i(X_t)}{t} \\
 & + K^{3/2}\sqrt{2}\left[\frac{4 \max_i C_i}{\delta} 
 + \frac{4C_*}{\delta}\right]  \sqrt{\frac{\ln(\ln(t))-\ln(2K)}{\ln(t)}} 
\end{align*}
from where (\ref{eq:mainth2}) follows. To prove the lower bound, we have, for $i$ not optimal, that  
\begin{align*}
    \E_{X_t} [ \Delta_i(X_t) \bbP(I_t=i | X_t) ] 
    &\ge \E_{X_t} \Delta_i(X_t) \epsilon_t/K  \\
    &\ge \E_{X_t} \delta \epsilon_t/K  
    \ge \delta /(t K).
\end{align*}
Then using the suboptimal action, assumed to exist, we get 
\begin{align*}
 \E_{X_t} & \E_{I_t} \E_R[R_*(X_t) - R(X_t)] \\
& = \sum_{i} \E_{X_t} \Delta_i(X_t) \bbP(I_t=i | X_t) 
 \ge \delta /(tK).  
\end{align*}
\end{proof}

\begin{corollary} \label{generalization}
The Epsilon Greedy method with any predictor, neural network or otherwise, with convergence of $ c \sqrt{\frac{\ln(n)}{n}}$, or better, will have regret converging to 0 almost surely.
\end{corollary}

\begin{remark}
With $\epsilon_t = 1/t^p$ with $p \le 1$, enough samples will be taken to train an approximation to convergence. When $p > 1$, The number of samples is finite and the approximation will not converge in general. This is called a starvation scenario since the optimal action is not sampled sufficiently. 
\end{remark}

\begin{corollary} 
\label{opt_p}
The optimal $p$ for $\epsilon_t = 1/t^p$ with the fastest converging upper bound of Theorem \ref{thm:epsilon_greedy_bound} for Deep Epsilon Greedy is $p=1/3$.
\end{corollary}

\begin{proof}[Proof of Corollary \ref{opt_p}]
From the above remark, we know $p\le 1$. First we show that $0 < p < 1$ converges faster than $p=1$. Theorem \ref{thm:epsilon_greedy_bound} gives the bound of 
$$ \frac{\max_i \E_{X_t} \Delta_i(X_t)}{t} 
 + K^{3/2}\frac{C_0}{\delta} \sqrt{\frac{\ln(\ln(t))-\ln(2K)}{\ln(t)}}$$ 
 for $p=1$. Theorem \ref{thm:epsilon_greedy_bound_p} gives the bound of 
\begin{equation} \label{eq:pNot1}
 \frac{\max_i \E_{X_t} \Delta_i(X_t)}{t^p}
 + K^{3/2}\frac{C_0'}{\delta}  
 \sqrt{\frac{\ln(t^{-p+1})-\ln(2(-p+1)K)}{t^{-p+1}}}
\end{equation}
for $p<1$. 
Set $\alpha_t := \max_i \E_{X_t} \Delta_i(X_t)$ 

and 
$\beta_t = K^{3/2}\frac{C_0}{\delta} \sqrt{\ln(\ln(t))-\ln(2K)}$ 

and 
$\gamma_t = K^{3/2}\frac{C_0'}{\delta} \sqrt{\ln(t^{-p+1})-\ln(2(1-p)K)}$. Using the ratio test, the limit of the ratio of the bounds is 
\begin{align*}
 &\lim_{t\rightarrow\infty} \frac{\alpha_t t^{-p} + \gamma_t t^{(p-1)/2} }
 { \alpha_t t^{-1} + \beta_t / \sqrt{\ln(t)}} \\
 &=  \frac{ \lim_{t\rightarrow\infty}
\alpha_t t^{-p} \beta_t^{-1} \sqrt{\ln(t)} + \gamma_t \beta_t^{-1}t^{(p-1)/2} \sqrt{\ln(t)} } 
 { \lim_{t\rightarrow\infty} \alpha_t t^{-1} \beta_t^{-1} \sqrt{\ln(t)} + 1 } \\
 & = \lim_{t\rightarrow\infty} t^{(p-1)/2} \sqrt{\ln(t)} = 0 
\end{align*}
for $0 < p < 1$. Now we find $p$ to minimize Equation \eqref{eq:pNot1}. 
$$ \lim_{t\rightarrow\infty} \frac{\max_i \E_{X_t} \Delta_i(X_t)}{t^p} + K^{3/2}\frac{C_0}{\delta} \sqrt{\frac{\ln(t^{-p+1}) -\ln(2(1-p)K)}{t^{-p+1}}} $$
$$ = t^{-p} \lim_{t\rightarrow\infty} \alpha_t 
+ K^{3/2}\frac{C_0}{\delta} \sqrt{\frac{\ln(t^{-p+1}) -\ln(2(1-p)K)}{ t^{-3p+1}}} $$
The convergence rate is $t^{-p}$ if $-3p+1 \ge 0$. Otherwise, $p>1/3$, the rate achieved is $t^{-(1-p)/2}>t^{-(1-1/3)/2}=t^{-1/3}$. So the optimal is at $p=1/3$. 

\end{proof}

\section{MNIST Experiments}
We experiment with the MNIST \cite{mnist} dataset which contains real world, handwritten digits 0-9. The action set is $\{a_1,...,a_5\}$. The state, $X_t$ at each time step, $t$, is 5 random images. Each digit has an equal chance of being chosen. The reward $Y_t=R(X_t,I_t)$ is the digit of the image corresponding to the chosen action plus Gaussian noise. We plot the regret convergence to 0 of the Deep Epsilon Greedy method in Figure \ref{fig_convergence}. We get a convergence rate of approximately $t^{-1/2}$ which is within the bounds of Theorem \ref{thm:epsilon_greedy_bound}. Next, we compare the Uniform Random method (see Algorithm \ref{algo:random}), the Linear Regression method (see Algorithm \ref{algo:linear}), the LinUCB method \cite{li_contextual-bandit_2010} (see Algorithm \ref{algo:linucb}), the Deep Epsilon Greedy method (see Algorithm \ref{algo:deep_ep_greedy}), and the Simple Deep Epsilon Greedy method (only 1 hidden layer). The Deep Epsilon Greedy method uses 3 convolutional layers followed by a fully connected layer of width 100. The Simple Deep Epsilon Greedy method uses just one fully connected layer of width 100. For all methods, training is every 20 time steps. For the neural networks, number of training epochs is always 16 and the initial learning rate is $10^{-3}$. We plot the reward normalized (divided by the time step) in Figures \ref{fig_low} and \ref{fig_high}. Each curve is an average of 12 independent runs. 

\begin{figure}[t]
\begin{center}
\centerline{
\includegraphics[scale=.6]{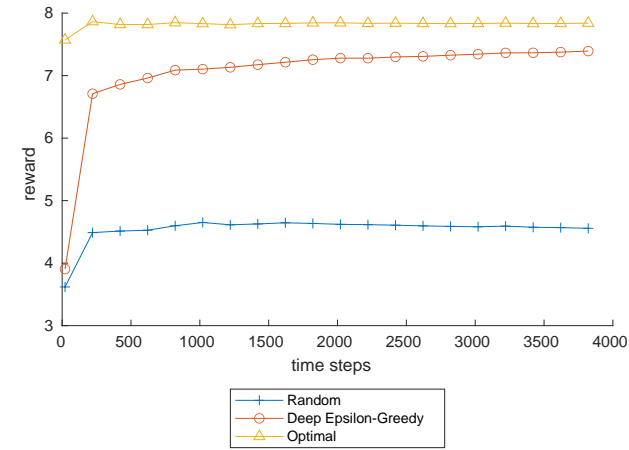}
}
\centerline{
\includegraphics[scale=.6]{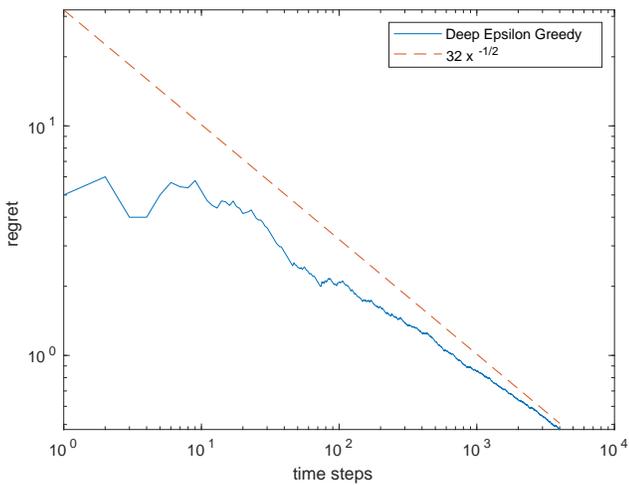}
}
\caption{Deep Epsilon Greedy method convergence of regret to 0 at rate $x^{-1/2}$. Plotting normalized reward of optimal method minus normalized reward of Deep Epsilon Greedy method. No noise added to MNIST dataset. Single run with 1000 neurons in the fully connected, final layer. }
\label{fig_convergence}
\end{center}
\end{figure}

We see that in both the high noise and low noise case, Deep Epsilon Greedy converges to the optimal but Simple Deep Epsilon Greedy does not. Simple Deep Epsilon does not have the necessary complexity to converge required by Theorem \ref{nnet}. Because the neural network does not converge, the regret in this policy does not converge to 0. The LinUCB and Linear Regression also cannot converge to the solution because they are linear models but this is a nonlinear problem. So they perform as well as purely random actions. 

\begin{figure}[H]
\begin{center}
\centerline{\includegraphics[scale=.6]{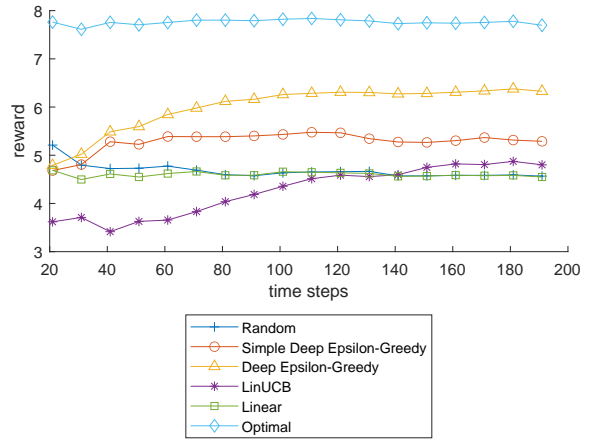}}
\caption{Low Noise: Mean reward normalized (divide by time step) plotted over time steps for each method. Task is to choose the largest MNIST image (digit) of 5 random images. No Gaussian noise added to reward. Mean is over 12 independent runs. }
\label{fig_low}
\end{center}
\end{figure}

\begin{figure}[H]
\begin{center}
\centerline{\includegraphics[scale=.6]{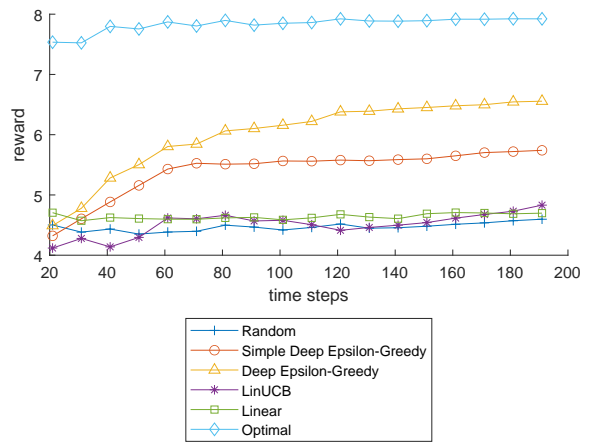}}
\caption{High Noise: Mean reward normalized (divide by time step) plotted over time steps for each method. Task is to choose the largest MNIST image (digit) of 5 random images. Gaussian noise, sigma = 1, added to reward. Mean is over 12 independent runs. }
\label{fig_high}
\end{center}
\end{figure}

\section{Summary}
We have shown convergence guarantees for the Deep Epsilon Greedy method, Algorithm \ref{algo:deep_ep_greedy}. In Corollary \ref{generalization}, we have shown convergence of generalizations to other common predictive models. We have shown convergence failure with $\epsilon_t = 1/t^p$ for $p>1$. In Corollary \ref{opt_p}, we showed that $\epsilon_t = 1/t^{1/3}$\ gives the fastest convergence bound. To see these results in experiments, we perform a standard MNIST \cite{mnist} experiment. The converging methods vs non-converging methods is empirically confirmed and displayed in Figures \ref{fig_convergence}, \ref{fig_low}, and \ref{fig_high}.



\bibliography{example_paper}

\begin{thebibliography}{12}
\providecommand{\natexlab}[1]{#1}
\providecommand{\url}[1]{\texttt{#1}}
\expandafter\ifx\csname urlstyle\endcsname\relax
  \providecommand{\doi}[1]{doi: #1}\else
  \providecommand{\doi}{doi: \begingroup \urlstyle{rm}\Url}\fi

\bibitem[Agrawal \& Goyal(2012)Agrawal and Goyal]{agrawal_analysis_2012}
Agrawal, S. and Goyal, N.
\newblock Analysis of thompson sampling for the multi-armed bandit problem.
\newblock \emph{{JMLR} Workshop and Conference Proceedings}, pp.\ ~26, 2012.

\bibitem[Athey \& Wager(2021)Athey and Wager]{athey_policy_2020}
Athey, S. and Wager, S.
\newblock Policy learning with observational data.
\newblock \emph{Econometrica}, 89\penalty0 (1):\penalty0 133--161, 2021.

\bibitem[Auer et~al.(2002)Auer, Cesa-Bianchi, and
  Fischer]{auer_finite-time_2002}
Auer, P., Cesa-Bianchi, N., and Fischer, P.
\newblock Finite-time analysis of the multiarmed bandit problem.
\newblock \emph{Machine Learning}, 47\penalty0 (2):\penalty0 235--256, 2002.
\newblock ISSN 08856125.
\newblock \doi{10.1023/A:1013689704352}.

\bibitem[Chen et~al.(2020)Chen, Liu, Liao, and Zhao]{chen_doubly_2020}
Chen, M., Liu, H., Liao, W., and Zhao, T.
\newblock Doubly robust off-policy learning on low-dimensional manifolds by
  deep neural networks.
\newblock \emph{Submitted to Operations Research, under revision}, 2020.

\bibitem[Györfi et~al.(2002)Györfi, Kohler, Krzyżak, and
  Walk]{gyorfi_distribution-free_2002}
Györfi, L., Kohler, M., Krzyżak, A., and Walk, H.
\newblock \emph{A Distribution-Free Theory of Nonparametric Regression}.
\newblock Springer Series in Statistics. Springer New York, 2002.
\newblock ISBN 978-0-387-95441-7 978-0-387-22442-8.
\newblock \doi{10.1007/b97848}.

\bibitem[Li et~al.(2010)Li, Chu, Langford, and
  Schapire]{li_contextual-bandit_2010}
Li, L., Chu, W., Langford, J., and Schapire, R.~E.
\newblock A contextual-bandit approach to personalized news article
  recommendation.
\newblock \emph{Proceedings of the 19th international conference on World wide
  web - {WWW} '10}, pp.\  661, 2010.
\newblock \doi{10.1145/1772690.1772758}.

\bibitem[Rawson \& Freeman(2021)Rawson and Freeman]{rawson_deep_2021}
Rawson, M. and Freeman, J.
\newblock Deep upper confidence bound algorithm for contextual bandit ranking
  of information selection.
\newblock \emph{Proceedings of Joint Statistical Meetings (JSM), Statistical
  Learning and Data Science Section, 2021}, 2021.

\bibitem[Sutton \& Barto(1998)Sutton and Barto]{sutton_reinforcement_1998}
Sutton, R.~S. and Barto, A.~G.
\newblock \emph{Reinforcement Learning: An Introduction}.
\newblock Cambridge, {MA}, 1998.

\bibitem[Xu et~al.(2021)Xu, Yang, Wang, and Liang]{xu2021doubly}
Xu, T., Yang, Z., Wang, Z., and Liang, Y.
\newblock Doubly robust off-policy actor-critic: Convergence and optimality.
\newblock \emph{arXiv preprint arXiv:2102.11866}, 2021.

\bibitem[{Yann LeCun}(2021)]{mnist}
{Yann LeCun}.
\newblock The mnist database of handwritten digits.
\newblock \url{http://yann.lecun.com/exdb/mnist/}, 2021.

\bibitem[Zhou et~al.(2020)Zhou, Li, and Gu]{zhou2020neural}
Zhou, D., Li, L., and Gu, Q.
\newblock Neural contextual bandits with ucb-based exploration.
\newblock In \emph{Proceedings of the 37th International Conference on Machine
  Learning}, volume 119, pp.\  11492--11502. PMLR, 13--18 Jul 2020.

\bibitem[Zou et~al.(2019)Zou, Balan, and Singh]{zou2019lipschitz}
Zou, D., Balan, R., and Singh, M.
\newblock On lipschitz bounds of general convolutional neural networks.
\newblock \emph{IEEE Transactions on Information Theory}, 66\penalty0
  (3):\penalty0 1738--1759, 2019.
\newblock \doi{10.1109/TIT.2019.2961812}.

\end{thebibliography}
\bibliographystyle{icml2022}

\appendix
\onecolumn
\twocolumn
\section*{Appendix}

\begin{proof}[Proof of lemma \ref{lemma:epsilon_greedy_bound}, part II.]
Fix $i$. Recall $\epsilon_t = 1/t^p$. Following the proof of theorem 3 in \cite{auer_finite-time_2002}, 
\begin{align*}
 \E( T^R_i(t) ) &= \sum_{l=1}^{t-1} \bbP(\eta < \epsilon_l \wedge \rho = i) 
 = \sum_{l=1}^{t-1} \bbP(\eta < \epsilon_l) \bbP(\rho = i) \\
 &= \sum_{l=1}^{t-1} \epsilon_l/K 
 =\frac{1}{K} \sum_{l=1}^{t-1} \frac{1}{l^p} 
 \ge \frac{1}{(1-p) K} \ t^{1-p}
\end{align*}
Set $B(t):=\frac{1}{K} \sum_{l=1}^{t-1} \frac{1}{l^p} $. we have
\begin{align*}
Var(T^R_i(t)) &= \sum_{l=1}^{t-1} \frac{\epsilon_l}{K}(1-\frac{\epsilon_l}{K}) 
\le \frac{1}{K} \sum_{l=1}^{t-1} \epsilon_l \\
&=\frac{1}{K} \sum_{l=1}^{t-1} \frac{1}{l^p} 
=B(t).
\end{align*}
By Bernstein’s inequality 
\begin{align*}
 \bbP(&T^R_i(t)  \le B(t)/2 ) 
 = \bbP\left(T^R_i(t) - B(t) \le -B(t)/2 \right) \\
&\le \exp\left(\frac{-B(t)^2/8}{Var(T^R_i(t))+\frac{1}{3} B(t)/2}\right)\\ 
&\le \exp\left(\frac{-B(t)^2/8}{B(t)+\frac{1}{3} B(t)/2}\right) \\
&\le \exp\left(-\frac{3 B(t)}{28}\right)
 \le \exp\left(-\frac{3 \ t^{-p+1}}{28 (-p+1) K}\right). 
\end{align*}
So by union bound
$$ \bbP\left(\bigvee_{i=1}^K \{T^R_i(t) 
 \le \frac{t^{-p+1}}{2(-p+1) K} \}\right) $$
$$ \le K \ \bbP\left( T^R_1(t) \le \frac{t^{-p+1}}{2(-p+1) K}\right) $$
$$ \le K \ \bbP\left( T^R_1(t) \le \frac{B(t)}{2}\right) $$
$$ \le K\exp\left(-\frac{3 \ t^{-p+1}}{28 (-p+1) K}\right). $$
And
$$ \bbP\left(\bigwedge_{i=1}^K \{T^R_i(t) \ge \frac{t^{-p+1}}{2(-p+1) K}\}\right) $$
$$ \ge 1-K\exp\left(-\frac{3 \ t^{-p+1}}{28 (-p+1) K}\right). $$
This prove Equation \eqref{lemma:epsilon_greedy_bound2}. A similar proof for Equation \eqref{lemma:epsilon_greedy_bound1} is above. 
\end{proof}

\begin{proof}[Proof of theorem \ref{thm:epsilon_greedy_bound_p}] 

Let $*$ be an optimal action at $X_t$ and $R$ the reward from the epsilon greedy method. 


Let $\Phi_{i,t}$ be the trained neural network for action i and have t parameters.
By lemma \ref{lemma:epsilon_greedy_bound}, with probability greater than 
$1-K\exp\left(-\frac{3 \ t^{-p+1}}{28 (-p+1) K}\right)$,\ 
$T^R_i(t) \ge \frac{t^{-p+1}}{2(-p+1) K}$ for all $i$. In this case, we have 

\begin{align*}
 \E_{X_t} & \E_{I_t} \E_R[R_*(X_t) - R(X_t)] \\
&  = \E_{X_t} [\mu_*(X_t) - \E_{I_t} \E_R R(X_t) ] \\
& = \E_{X_t} [\mu_*(X_t) - \sum_{i=1}^K \mu_i(X_t) \bbP(I_t=i | X_t) ] \\
& = \E_{X_t} \sum_{i} \Delta_i(X_t) \bbP(I_t=i | X_t) \\
& = \sum_{i} \E_{X_t} \Delta_i(X_t) \bbP(I_t=i | X_t) 
\end{align*}
Then
\begin{align*}
 \E_{X_t} [ & \Delta_i(X_t) \bbP(I_t=i | X_t) ] \\
 & \le \E_{X_t} \Delta_i(X_t) [\epsilon_t/K + \bbP(\Phi_{i,t}(X_t) \ge \Phi_{*,t}(X_t))] 
\end{align*}
and, with Markov's inequality,
\begin{align*}
\bbP ( &\Phi_{i,t}(X_t) \ge \Phi_{*,t}(X_t) ) \\
&\le \bbP(\Phi_{i,t}(X_t) \ge \mu_i(X_t) + \Delta_i(X_t)/2) + \\ 
& +\bbP(\Phi_{*,t}(X_t) \le \mu_*(X_t) - \Delta_i(X_t)/2) \\
& = \int_{\bbI\{\Phi_{i,t}(X_t) \ge \mu_i(X_t) + \Delta_i(X_t)/2\}} dP_i + \\
& + \int_{\bbI\{\Phi_{*,t}(X_t) \le \mu_*(X_t) - \Delta_i(X_t)/2\}} dP_* \\
& \le \int_{\bbI\{|\Phi_{i,t}(X_t) - \mu_i(X_t)| \ge \Delta_i(X_t)/2\}} dP_i + \\
& + \int_{\bbI\{|\Phi_{*,t}(X_t)-\mu_*(X_t)| \ge \Delta_i(X_t)/2\}} dP_* \\
& \le \int \frac{|\Phi_{i,t}(X_t) - \mu_i(X_t)|^2 }{\Delta_i(X_t)^2/4} dP_i \\
& +  \int \frac{|\Phi_{*,t}(X_t)-\mu_*(X_t)|^2 }{\Delta_i(X_t)^2/4} dP_* 
\end{align*}

Then
\begin{align*}
 \E_{X_t} [ & \Delta_i(X_t) \bbP(I_t=i | X_t) ] \\
& \le \E_{X_t} \Delta_i(X_t) \epsilon_t/K + \\
& + \E_{X_t} \Delta_i(X_t) \int \frac{|\Phi_{i,t}(X_t) - \mu_i(X_t)| }{\Delta_i(X_t)/2} dP_i + \\
& + \E_{X_t} \Delta_i(X_t) \int \frac{|\Phi_{*,t}(X_t)-\mu_*(X_t)| }{\Delta_i(X_t)/2} dP_* 
\end{align*}

\begin{align*}
& \le \E_{X_t} \Delta_i(X_t) \epsilon_t/K \\
& + \frac{4}{\delta} \int_{x_t : i\ne *} \int |\Phi_{i,t}(x_t) - \mu_i(x_t)|^2 dP_i dP_{x_t}  + \\
& + \frac{4}{\delta} \int_{x_t : i\ne *} \int |\Phi_{*,t}(x_t)-\mu_*(x_t)|^2 dP_* dP_{x_t},
\end{align*}
by dominated convergence,
\begin{align*}
& \le \E_{X_t} \Delta_i(X_t)\ \epsilon_t/K \\
& + \frac{4}{\delta} \int \int_{x_t } |\Phi_{i,t}(x_t) - \mu_i(x_t)|^2 dP_{x_t} dP_i + \\
& + \frac{4}{\delta} \int \int_{x_t } |\Phi_{*,t}(x_t)-\mu_*(x_t)|^2 dP_{x_t} dP_* 
\end{align*}

Recall that we are in the case that 
$T^R_i(t) \ge \frac{t^{-p+1}}{2(-p+1) K}$ for all $i$. Let $C_i$ be the constant from Theorem \ref{nnet} for neural network $i$ and let $n_i$ be the minimal value of the training data size such that equation \ref{eq:nnet} holds. Choose $t_0>(2 (1-p) K \max\{e,\max_i n_i\})^{1/(1-p)}$. Since the map $x \mapsto \sqrt{\frac{ln(x)}{x}}$ is monotone decreasing for $x>e$, the above expression is further upper bounded by 
\begin{align*}
& \le \E_{X_t} \Delta_i(X_t)\ \frac{\epsilon_t}{K}
 + \frac{4}{\delta}  C_i\sqrt{\frac{\ln(T_i(t))}{T_i(t)}} 
 + \frac{4}{\delta}  C_*\sqrt{\frac{\ln(T_*(t))}{T_*(t)}}  \\
& \le \E_{X_t} \Delta_i(X_t)\ \frac{\epsilon_t}{K}
 + \frac{4C_i}{\delta}  \sqrt{\frac{\ln(T^R_i(t))}{T^R_i(t)}} 
 + \frac{4C_*}{\delta}  \sqrt{\frac{\ln(T^R_*(t))}{T^R_*(t)}}  \\
& \le \frac{\E_{X_t} \Delta_i(X_t)}{K t^p} \\
& + \left[\frac{4C_i}{\delta}  
 + \frac{4C_*}{\delta}\right] 
 \sqrt{\frac{\ln(t^{-p+1}/(2(-p+1) K))}{t^{-p+1}/(2(-p+1) K)}}
\end{align*}
So
\begin{align*}
& \E_{X_t} \E_{I_t} \E_R [R_*(X_t) - R(X_t)] \\
& = \E_{X_t} [\mu_*(X_t) - \E_{I_t} \E_R R(X_t) ] \\
& = \sum_{i} \E_{X_t} \Delta_i(X_t) \bbP(I_t=i | X_t) \\
& \le \frac{\max_i \E_{X_t} \Delta_i(X_t)}{t^p} \\
& + K^{3/2}\sqrt{2(1-p)}\left[\frac{4 \max_i C_i}{\delta} 
 + \frac{4C_*}{\delta}\right] \\ 
& \cdot \sqrt{\frac{\ln(t^{-p+1})-\ln(2(-p+1)K)}{t^{-p+1}}} 
\end{align*}
from where (\ref{eq:epsilon_greedy_bound_p3}) follows. To prove the lower bound, we have, for $i$ not optimal, that  
\begin{align*}
    \E_{X_t} [ \Delta_i(X_t) \bbP(I_t=i | X_t) ] 
    &\ge \E_{X_t} \Delta_i(X_t) \epsilon_t/K  \\
    &\ge \E_{X_t} \delta \epsilon_t/K  
    \ge \delta /(K t^p).
\end{align*}
Then using the suboptimal action, assumed to exist, we get 
\begin{align*}
 \E_{X_t} & \E_{I_t} \E_R[R_*(X_t) - R(X_t)] \\
& = \sum_{i} \E_{X_t} \Delta_i(X_t) \bbP(I_t=i | X_t) 
 \ge \delta /(K t^p).  
\end{align*}
\end{proof}

\begin{algorithm}[h]
\SetAlgoLined
    \textbf{Input:}\\
    $M \in \mathbb{N}$ : Total time steps \\
    $m \in \mathbb{N}$ : Context dimension \\
    $ X \in \mathbb{R}^{M \times m}$ where state $X_{t}\in\mathbb{R}^m$ for time step $t$  \\ 
    $ A = \{action_1,...,action_K\}$ : Available Actions \\
    $ Reward : \mathbb{N}_{[1,K]} \rightarrow \mathbb{R}$  \\
    \textbf{Output:}\\
    $ D \in \N^M $ : Decision Record \\
    $ R \in \mathbb{R}^{M}$ where $R_{t}$ stores the reward from time step $t$  \\ 
    \textbf{Begin:}\\
    \For{t = 1, 2, ..., $M$}{
        $ \rho \sim $ Uniform(\{1,...,K\})  \\
        $D_{t} = A_{\rho}$ (Choose Random Action) \\
        $R_{t} = Reward(D_{t})$ 
    }
\caption{Random}
\label{algo:random}
\end{algorithm}

\begin{algorithm}[h]
\SetAlgoLined
    \textbf{Input:}\\
    $M \in \mathbb{N}$ : Total time steps \\
    $m \in \mathbb{N}$ : Context dimension \\
    $ X \in \mathbb{R}^{M \times m}$ where state $X_{t}\in\mathbb{R}^m$ for time step $t$  \\ 
    $ A = \{action_1,...,action_K\}$ : Available Actions \\
    $ Reward : \mathbb{N}_{[1,K]} \rightarrow \mathbb{R}$  \\
    $ oracle : \mathbb{N}_{[1,M]} \rightarrow \mathbb{N}_{[1,K]}$ : Oracle for correct action index \\
    \textbf{Output:}\\
    $ D \in \N^M $ : Decision Record \\
    $ R \in \mathbb{R}^{M}$ where $R_{t}$ stores the reward from time step $t$  \\ 
    \textbf{Begin:}\\
    \For{t = 1, 2, ..., $M$}{
        $D_{t} = oracle(t)\ $ (Choose Correct Action) \\
        $R_{t} = Reward(D_{t})$ 
    }
\caption{Optimal}
\label{algo:optimal}
\end{algorithm}

\begin{algorithm}[h]
\SetAlgoLined
    \textbf{Input:}\\
    $M \in \mathbb{N}$ : Total time steps \\
    $m \in \mathbb{N}$ : Context dimension \\
    $ X \in \mathbb{R}^{M \times m}$ where state $X_{t}\in\mathbb{R}^m$ for time step $t$  \\ 
    $ A = \{action_1,...,action_K\}$ : Available Actions \\
    $ Reward : \mathbb{N}_{[1,K]} \rightarrow \mathbb{R}$  \\
    \textbf{Output:}\\
    $ B_j \in \R^{m} : $ Linear Models for $1\le j \le K$ \\
    $ D \in \N^M $ : Decision Record \\
    $ R \in \mathbb{R}^{M}$ where $R_{t}$ stores the reward from time step $t$  \\ 
    \textbf{Begin:}\\
    \For{j = 1, 2, ..., K}{
        $B_j = 0$
    }
    \For{t = 1, 2, ..., $M$}{
        \For{j = 1, 2, ..., K}{
            $ \hat \mu_{action_j} = B_j^T\ X_{t} $ (predict rewards)
        }
        $D_{t} =  \arg \max_{1 \le j \le K } \ \hat \mu_{action_j} $ \\
        $R_{t} = Reward(D_{t})$ \\
        (Training Stage)\\
        $ S = \{l : 1\le l\le t,\ D_l = D_t\}$ \\
        $ B_{D_t} = \arg\min_B \|R_{S} - B X_{S}^T\|_2$
    }
\caption{Linear}
\label{algo:linear}
\end{algorithm}

\begin{algorithm}[h]
\SetAlgoLined
    \textbf{Input:}\\
    $M \in \mathbb{N}$ : Total time steps \\
    $m \in \mathbb{N}$ : Context dimension \\
    $ X \in \mathbb{R}^{M \times m}$ where state $X_{t}\in\mathbb{R}^m$ for time step $t$  \\ 
    $ A = \{action_1,...,action_K\}$ : Available Actions \\
    $ Reward : \mathbb{N}_{[1,K]} \rightarrow \mathbb{R}$  \\
    \textbf{Output:}\\
    $ B_j \in \R^{m \times m} : $ Linear Maps for $1\le j \le K$ \\
    $ b_j \in \R^{m} : $ Linear Models for $1\le j \le K$ \\
    $ D \in \N^M $ : Decision Record \\
    $ R \in \mathbb{R}^{M}$ where $R_{t}$ stores the reward from time step $t$ \\ 
    \textbf{Begin:}\\
    \For{j = 1, 2, ..., K}{
        $B_j = I$ \\
        $b_j = 0$
    }
    \For{t = 1, 2, ..., $M$}{
        \For{j = 1, 2, ..., K}{
            $\Theta = B^{-1}_j b_j $ \\
            (Predict Rewards)\\
            $ \hat \mu_{action_j} = \Theta^T \ X_{t} + \sqrt{ X_t^T B^{-1}_j X_t }\ $ 
        }
        $D_{t} =  \arg \max_{1 \le j \le K } \ \hat \mu_{action_j} $ \\
        $R_{t} = Reward(D_{t})$ \\
        (Training Stage)\\
        $ B_{D_t} = B_{D_t} + X_t X_t^T$ \\
        $ b_{D_t} = b_{D_t} + R_t X_t$
    }
\caption{LinUCB}
\label{algo:linucb}
\end{algorithm}


\end{document}